\documentclass[final]{article}

\usepackage{neurips_2021}
\usepackage{subfigure,tikz,pgf, graphicx}
\usepackage{epsfig,amssymb,amsmath,ifthen,comment,mathtools}
\usepackage{color}
\usepackage{array}
\usepackage{pdflscape}
\usepackage[figuresright]{rotating}
\usepackage{xr}

\usetikzlibrary{arrows,shapes.arrows,calc,shapes.geometric,shapes.multipart, decorations.pathmorphing}
\usetikzlibrary{shapes,decorations,arrows,calc,arrows.meta,fit,positioning}
\tikzset{
	-Latex,auto,node distance =1 cm and 1 cm,semithick,
	state/.style ={ellipse, draw, minimum width = 0.7 cm},
	point/.style = {circle, draw, inner sep=0.04cm,fill,node contents={}},
	bidirected/.style={Latex-Latex,dashed},
	el/.style = {inner sep=2pt, align=left, sloped}
}

\usepackage[utf8]{inputenc} 
\usepackage[T1]{fontenc}    
\usepackage{hyperref}       
\usepackage{url}            
\usepackage{booktabs}       
\usepackage{amsfonts}       
\usepackage{nicefrac}       
\usepackage{microtype}      
\usepackage{xcolor}         

\newcommand{\info}{\mathrm{\sigma}}

\newcommand\SCM{\textbf{SCM}}

\newcommand{\BlackBox}{\rule{1.5ex}{1.5ex}}  
\newenvironment{proof}{\par\noindent{\bf Proof\ }}{\hfill\BlackBox\\[2mm]}
\newtheorem{example}{Example} 
\newtheorem{theorem}{Theorem}

\newtheorem{definition}[theorem]{Definition}

\title{Path-specific Effects Based on Information Accounts of Causality}

\author{%
  Gong Heyang \\
  Department of Statistics and Finance\\
  University of Science and Technology of China\\
  \texttt{zj3712@gmail.com} \\
  \And
  Zhu Ke \\
  Department of Statistics and Actuarial Science \\
  University of HongKong\\
  \texttt{mazhuke@hku.hk} \\
  
}

\begin{document}

\maketitle

\begin{abstract}
Path-specific effects in mediation analysis provide a useful tool for fairness analysis, which is mostly based on nested counterfactuals. However, the dictum ``no causation without manipulation'' implies that path-specific effects might be induced by certain interventions. This paper proposes a new path intervention inspired by information accounts of causality, and develops the corresponding intervention diagrams and $\pi$-formula. 
Compared with the interventionist approach of \citet{robins2020interventionist} based on nested counterfactuals, our proposed path intervention method explicitly describes the manipulation in structural causal model with a simple information transferring interpretation, and does not require the non-existence of recanting witness to identify path-specific effects. Hence, it could serve useful communications and theoretical focuses for mediation analysis.
\end{abstract}

\section{Introduction}


Analyzing the relative strength of different pathways between a decision $X$ and an outcome $Y$ is an interesting topic for both scientists and practitioners across many disciplines over several decades. The path-specific effects in mediation analysis have been demonstrated as a useful tool to tackle this topic \citep{baron1986moderator, robins1992identifiability, pearl2001direct}. This is one of the seven tasks summarized by Judea Pearl in causal inferences \citep{pearl2019seven}. To capture the path-specific effects, the often used approaches are based on nested counterfactuals; see \citet{avin2005identifiability}, \citet{shpitser2016causal},  \citet{zhang2018non}, and \citet{malinsky2019potential} for earlier references. 

More recently, \citet{robins2020interventionist} provided an interventionist theory of mediation to study some path-specific effects  (such as natural indirect/direct effects) in single world intervention graph (SWIG). However, their interventionist framework seems lack of a formally defined concept of intervention for the general path-specific effects, and it requires the recursive assumption to well define the nested counterfactuals. Also, their path specific counterfactuals are equal to interventional counterfactuals only when no recanting witness exists. Therefore, their method are restrictive in many situations. For example, 
to analyze the strength of path $\pi: X \to A \to Y$ between a decision $X$ and an outcome $Y$ in Fig. \ref{fig:base}(a) or (b), their method is not applicable any more.


\begin{figure}[http]
\label{fig:base}
	\begin{center}
		\begin{tikzpicture}[>=stealth, node distance=1.2cm]
		\tikzstyle{format} = [draw, thick, circle, minimum size=4.0mm,
		inner sep=1pt]
		\tikzstyle{unode} = [draw, thick, circle, minimum size=1.0mm,
		inner sep=0pt,outer sep=0.9pt]
		\tikzstyle{square} = [draw, very thick, rectangle, minimum size=4mm]

	\begin{scope}[xshift=0cm]
		\path[->,  line width=0.9pt]
		node[format, shape=ellipse] (a) {$A$}
		node[format, shape=ellipse, above of=a] (m) {$M$}
		node[format, shape=ellipse, right of=a] (y) {$Y$}
		node[format, shape=ellipse, left of=a] (x) {$X$}
        (x) edge[red] (a)
		(a) edge[blue] (m)
		(m) edge[blue] (y)
		(a) edge[red] (y)
		node[below of=a]{(a)}	
		;
	\end{scope}

	\begin{scope}[xshift=6.0cm]
    	\path[->,  line width=0.9pt]
        node[format, shape=ellipse] (x) {$X$}
    	node[format, shape=ellipse, right of=x] (a) {$A$}
    	node[format, shape=ellipse, right of=a] (y) {$Y$}
    	(x) edge[red] (a)
    	(a) edge[red] (y)
    	(y) edge[blue, bend right=60] (a)
    	node[below of=a]{(b)}
    	;
	\end{scope}
		
		\end{tikzpicture}
		\end{center}
\caption{ (a) A causal diagram with a recanting witness $A$ for the path $\pi: X \to A \to Y$; (b) A cyclic causal diagram with a loop between $A$ and $Y$.}
\end{figure}
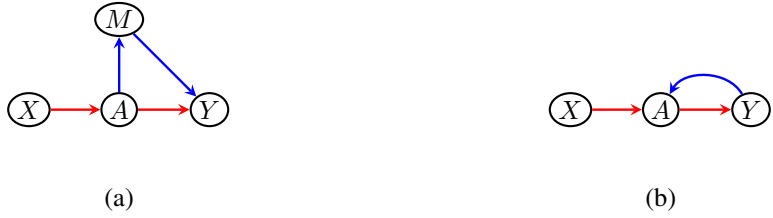

In this paper, we propose a path intervention based on the information accounts of causality. To understand causality as information transfer, which has been formally proposed and developed in philosophy since \citet{collier1999causation},  we decompose causal mechanisms into two parts: information transfer mechanisms and information process mechanisms. Specifically, each causal mechanism $f_i$ in a structural causal model (\SCM) decides how to process the collected information $e_{pa(i), i}$ from its input edges, where each $e_{j, i}$ decides the information transferred from its input node $j$ through the edge $j \to i$. Our novel path intervention is formally defined by explicitly manipulating mechanisms in the informational decomposition of \SCM. Henceforth, the path specific effects can be well defined by using our 
path intervention,  even for cyclic causal diagrams such as that in Fig. \ref{fig:base}(b), 

Moreover, we develop a $\pi$-formula for identifying our interventional counterfactuals. The use of the  $\pi$-formula does not need 
the absence of recanting witness, and hence it can be applied to identify the path-specific effects in Fig. \ref{fig:base}(a). 
As a graphical tool to assist causal reasoning,  we also propose the associated intervention diagram for different interventions under the informational decomposition of \SCM. Our path-specific effects can be interpreted as simply making the decision information $X = x'$ exclusively pass through the causal path $\pi$ while allowing all other information outside $\pi$ to behave naturally. This interpretation is different from those in 
 the existing graphical approaches based on nested counterfactuals.

The remaining paper is organized as follows.
Section \ref{sec:info} gives the preliminaries and the formal definition of informational decomposition of \SCM. Section \ref{sec:path} introduces our path intervention and intervention diagrams with some illustrating examples. Section \ref{sec:ident} proposes the identification formula for path-specific effects by using our path intervention. Concluding remarks and discussions are offered in Section \ref{sec:conclude}.

\paragraph{Notations.} Our notations below are similar to those in \citet{shpitser2020multivariate}. Fix a set of indices $V \equiv \{1, ..., n\}$. For each index $i \in V$, associate a random variable $X_i$ with state space $\mathcal{X}_i$. Given $A \subseteq V$, we will denote subsets of random variables indexed by $A$ as $X_A \in \mathcal{X}_A \equiv \prod_{i\in A} \mathcal{X}_i$. For notational conciseness, we will sometimes use index sets to denote random variables themselves, that is, using $V$ and $A$ to denote $X_V$ and $X_A$, respectively, and similarly using lower case $x_A$ to denote $x_A \in \mathcal{X}_A$. Similarly, by extension, we will also use $V_A$ to denote $X_A$ and $V_i$ to denote $X_i$. Finally, $M$ is usually used as a mediator, and $\mathcal{M}$ is an \SCM.

\section{Preliminaries and Informational Decomposition of SCM}

\label{sec:info}

The centre of modern causal modeling theory lies the structural causal model (SCM) (also known as structural equation model) that makes graphical assumptions of the underlying data generating process.\footnote{Another popular causal modeling framework is the potential-outcome framework, and one can refer to \citet{imbens2020potential} for more details.} There are many different formulations of SCM in the literature (see, e.g., \cite{bongers2016theoretical, pearl2009causal, pearl2019causal, scholkopf2019causality}), among which we use the definition in \citet{bareinboim2020pearl} in this paper.

\begin{definition}[SCM]
	\label{SCM-def}
	A structural causal model (SCM) $\mathcal{M}$  is a 4-tuple $(\mathbf{U}, \mathbf{V}, \mathcal{F}, P(\mathbf{U}))$, where
	\begin{itemize}
		\setlength\itemsep{0em}
		\item $\mathbf{U}$ is a set of background variables, also called exogenous variables, that are determined by factors outside the model;
		\item $\mathbf{V}$ is a set $\{V_1, V_2, ..., V_n\}$ of variables, called endogenous, that are determined by other variables in the models- that is, variables in $\mathbf{U} \cup \mathbf{V}$;
        \item $\mathcal{F}$ is a set of functions $\{f_1, f_2, ..., f_n\}$ such that each $f_i$ is a mapping from(the respective domains of) $U_i \cup Pa_i$ to $V_i$, where $U_i \subseteq \mathbf{U}, Pa_i \subseteq \mathbf{V} \setminus V_i$, and the entire set $\mathcal{F}$ forms a mapping from $\mathbf{U}$ to $\mathbf{V}$. That is, for $i \in V = \{1, ..., n\}$, each $f_i \in \mathcal{F}$ is such that
        \begin{equation}
            v_i \leftarrow f_i(pa_i, u_i),
        \end{equation}
        \item $P(\mathbf{U})$ is a probability function defined over the domain of $\mathbf{U}$.
	\end{itemize}
\end{definition}

Mechanisms are activities and entities organized to produce some phenomena which are represented by structural equations in an \SCM. A structural model is Markovian if the exogenous parent set $U_i, U_j$ are independent whenever $i \neq j$, and the associated causal graph for a Markovian SCM $\mathcal{M} = (\mathbf{U}, \mathbf{V}, \mathcal{F}, P(\mathbf{U}))$ can be defined as follows:

\begin{definition}[Causal Diagrams] $\mathcal{G}(\mathcal{M})=(V, E)$ is said to be a \emph{causal diagram} (of $\cal M$) if constructed as follows:
\begin{itemize}
	\setlength\itemsep{0em}
    \item add a vertex for every endogenous variable in the set $V$,
    \item add an edge $(V_j \to V_i)$ for every $i \in V$ if  $V_j$ appears as an argument of $f_i \in  \cal F$ in the edge set $E$.
\end{itemize}
\end{definition}

Various difference-making and production accounts in philosophy \footnote{See more details on \cite{illari2014causality}.} have been taken account by Judea Pearl for developing the structural causal modeling framework. However, the information accounts of causality, one of major approaches to study causality, has been comparatively neglected. The view of causality as information transfer was first proposed in \cite{collier1999causation}, further developed by \cite{illari2011theories} and \cite{illari2014causality}, and recently applied in neuroscience by \cite{feltz2019free}. Inspired by this view of causality, we propose the informational decomposition of \SCM \, in the following:

\begin{definition}[Informational Decomposition of \SCM]
	\label{info-SCM}
	Consider an \SCM $\mathcal{M} = (\mathbf{U}, \mathbf{V}, \mathcal{F}, P(\mathbf{U}))$. The \emph{informational decomposition} of $\mathcal{M}$ is defined as, for any $i\in V$, 
    \begin{equation}\label{eqn1}
    \left\{
    \begin{aligned}
    v_{i} &\leftarrow f_i(e_{pa(i), i}, u_i), \\
    e_{j, i} &\leftarrow v_j,
    \end{aligned}
    \right.
    \end{equation}	
    where $e_{j, i}$ represents the information on edge $(j, i)$ received from its input node $V_j$. \footnote{$pa_i$ usually stands for functional parents of node $i$ in an \SCM, while $pa(i)$ stands for parents in a causal diagram.} 
\end{definition}

Essentially, it can be interpreted as separating every causal mechanism $f_i$ into two part --- information process and information transfer. Mechanisms could be seen as information channels which is one centre idea of the informational accounts\citep{collier1999causation, illari2014causality}. Under this information transferring interpretation of \SCM, both \emph{do} intervention\citep{pearl1995causal} and \emph{info} intervention \citep{heyang2019info} are defined as following:

\begin{definition}[\emph{Do} and \emph{info} intervention]
	\label{def:intervention}
	Given an SCM $\mathcal{M} = (\mathbf{U}, \mathbf{V}, \mathcal{F}, P(\mathbf{U}))$ and any $I \subseteq V, v'_I \in \mathcal{V}_I$, the the \emph{do intervention} $do(v'_I)$ can be defined as a modification of structural equations to
    \begin{equation}
    \left\{
    \begin{aligned}
    v_{i} &\leftarrow f_i(e_{pa(i), i}, u_i) \, \text{if} \, i \notin I \, \text{else} \, v'_i  \\
    e_{j, i} &\leftarrow v_j 
    \end{aligned}
    \right.
    \end{equation}	
    while keeping everything else constant.
    And \emph{info intervention} $\info(V_I=v'_I)$ (or, in short, $\info(v'_I)$) maps $\mathcal{M}$ to the info-intervention model $\mathcal{M}^{\sigma(v'_I)}$  with a modified mechanism for every $i \in V$:
    \begin{equation}
    \left\{
    \begin{aligned}
    v_{i} &\leftarrow  f_i({e}_{pa(i), i}, u_i) \\
    {e}_{j, i} &\leftarrow v_j \, \text{if} \, j \notin I \, \text{else} \, v'_j  
    \end{aligned}
    \right.
    \end{equation}	
    while keeping everything else constant.
\end{definition}

The informational decomposition of \SCM \, reflect important epistemological distinctions about causality and lead to an alternative approach for mediation analysis. Comparing to the \emph{do} intervention, the \emph{info} intervention changes the information transferring instead of information processing mechanisms, which facilitates communication and theoretical focus as complementing causal modeling in terms of \emph{do} operator. This motivates us to proposed a novel path intervention as modification of information transfer mechanisms. In the following section, we propose a path intervention for path-specific effect based on the above informational decomposition.

\section{Path Intervention and Intervention Diagrams}
\label{sec:path}

Causal questions, such as what if we make something happen, can be formalized by using \emph{do} intervention. However, the \emph{do} intervention makes the intervention variable deterministic, by hypothetically forcing its value to a given constant. Henceforth, this causal notation is powerful but sometimes not enough for certain causal tasks, e.g., nature directed effects(NDE) defined by nested counterfactuals in mediation analysis. The dictum ``no causation without manipulation'' calls for an interventionist approach, recently addressed by \cite{robins2020interventionist} based on single world intervention graphs(SWIG) with certain limitation. Here, we are giving a explicit definition of path intervention based on informational decomposition of \SCM. 


\begin{definition}[Path intervention]
	\label{def:path-intervention}
	Consider a Makovian SCM $\mathcal{M} = (\mathbf{U}, \mathbf{V}, \mathcal{F}, P(\mathbf{U}))$. Then,
	the path intervention $\pi(A=a')$ (or, in short, $\pi({a'})$) along a causal path $\pi$ from $A$ to $Y$ maps $\mathcal{M}$ to the path-intervention model $\mathcal{M}^{\pi(a')}$
	with a modified causal mechanism for every $i \in V$:
	 \begin{equation}
	 \label{eq:path}
    \left\{
    \begin{aligned}
    v_{i} &\leftarrow  f_i({e}_{pa(i), i}, u_i), \\
    e_{j, i} &\leftarrow v_j, \\
    v^{\pi}_{i} &\leftarrow  f_i(e'_{pa(i), i}, u'_i), \\
    e'_{j, i} &\leftarrow v_j \, \text{if} \, (j, i) \notin \pi ;\, \text{else if } V_j \text{ is } A \,,  a'; \text{else} \, v^{\pi}_{j}, \\
    \end{aligned}
    \right.
    \end{equation}
 while keep everything else constant. Here, $v_i^{\pi}$ (standing for $v_i^{\pi(a')}$) and $v_i$ are in corresponding counterfactual and factual domains, respectively, $u'_i$ is a realization of ${U}'_i$, 
 and ${U}'_i$ is the $i$th entry of $\mathbf{U}'$
 which is an independent and identically distributed (i.i.d.) copy of $\mathbf{U}$.
\end{definition}

In $\mathcal{M}^{\pi(a')}$, we have a set of 
endogenous variables $\mathbf{V}\cup\mathbf{V}^{\pi(a')}$, where $\mathbf{V}^{\pi(a')}=\{V_1^{\pi(a')},V_2^{\pi(a')},...,V_n^{\pi(a')}\}$, and all $V_i^{\pi(a')}$ (or, in short, $V_i^\pi$) are referred as the effect variables of this path intervention $\pi(a')$. Hence, the path intervention in Def. \ref{def:intervention}  considers both factual and counterfactual worlds to capture the cross-world nature of path-specific effects. Moreover,  since the information $A=a'$ is only transferred to the descendants of $A$ along the path $\pi$ in 
$\mathcal{M}^{\pi(a')}$, it is easy to see that  the effect variables $V_i^{\pi} \stackrel{d}{=} V_i$ for $(j,i) \notin \pi$. In view of this fact, we  marginalize out the effect variables $V_i^{\pi}$ for $(j,i) \notin \pi$ in $\mathcal{M}^{\pi(a')}$.  Intuitively, our path intervention creates the counterfactual variables $V^{\pi}$ to transfer the information $A = a'$ exclusively through the causal path $\pi$ while allowing the information outside $\pi$ to remain unchanged.


We use use the NIE(natural indirect effect) as a toy example to show how our path intervention related to existing literatures.

\begin{example}[Natural indirect effect]
\label{eg1}
For an SCM with three endogenous variables, including treatment $A$, mediator ${M}$, Outcome $Y$, satisfies
    \begin{equation}
    \left\{
    \begin{aligned}
    a &\leftarrow  f_A(u_A) \\
    m &\leftarrow f_M(a, u_M) \\
    y &\leftarrow f_Y(a, m, u_Y)
    \end{aligned}
    \right.
    \end{equation}
The natural indirect effect(NIE) of $Y$ on $A$ through the mediator ${M}$ depends on $Y(a, M(a'))$ --- a variable in which two different levels of $a$ are nested within the counterfactual for $Y$. 
\end{example}

For a causal path $\pi: A \to M \to Y$ with transferred information $A=a'$, the path intervention results in a set of modified equations
    \begin{equation}
    \left\{
    \begin{aligned}
    a &\leftarrow  f_A(u_A) \\
    m &\leftarrow f_M(a, u_M) \\
    y &\leftarrow f_Y(a, m, u_Y) \\
    a^\pi &\leftarrow  f_A(u'_A) \\
    m^\pi &\leftarrow f_M(e'_{AM}, u'_M); e'_{AM} \leftarrow a'\\
    y^\pi &\leftarrow f_Y(e'_{AY}, e'_{MY}, u'_Y); e'_{AY} \leftarrow a, e'_{MY} \leftarrow m^\pi
    \end{aligned}
    \right.
    \end{equation}
Since the causal mechanisms of $a$ and $a^\pi$ coincide with the same input distribution information, the distributions of corresponding variables are the same. Henceforth, we marginalize  $A^\pi$  out the model, results in a causal model of variables of $(A, M, Y, M^\pi, Y^\pi)$ with mechanisms:
\begin{equation}
\label{eg1:pi-info}
\left\{
\begin{aligned}
a &\leftarrow  f_A(u_A) \\
m &\leftarrow f_M(a, u_M) \\
y &\leftarrow f_Y(a, m, u_Y) \\
m^\pi &\leftarrow f_M(e'_{AM}, u'_M); e'_{AM} \leftarrow a'\\
y^\pi &\leftarrow f_Y(e'_{AY}, e'_{MY}, u'_Y); e'_{AY} \leftarrow a, e'_{MY} \leftarrow m^\pi
\end{aligned}
\right.
\end{equation}
We can derive a set of structural equations ruling out $e'_{\cdot\cdot}$: 
\begin{equation}
\label{eg1:pi}
\left\{
\begin{aligned}
a &\leftarrow  f_A(u_A) \\
m &\leftarrow f_M(a, u_M) \\
y &\leftarrow f_Y(a, m, u_Y) \\
m^\pi &\leftarrow f_M(a', u'_M)\\
y^\pi &\leftarrow f_Y(a, m^\pi, u'_Y)
\end{aligned}
\right.
\end{equation}
In the above example, the path specific effect variable $Y^\pi$ are different  from $Y(a, M(a'))$ in two aspects. First, NIE usually considers the contrast of outcome under two given different levels $a, a'$ of the treatment, then mathematically $Y(a, M(a')) \neq  Y(A, M(a'))=Y^{\pi(a')}$. Second, the interpretation is completely different for $Y^\pi$, which is the effect of $Y$ had the information $A=a'$ passed through the path $\pi: A \to M \to Y$ while keeping other transferred information unchanged. Beyond this example, our path specific effects can even be well-defined for causal diagram with cycle, e.g. Fig \ref{fig:base}(b), by the solution of an \SCM \, with modified structural equations (\ref{eq:path}). 


The recognition that there are mechanisms underlying the phenomena of interest, but we usually cannot determine them precisely, gives rise to the discipline of causal inference \cite{pearl2000models}. Virtually every approach to causal inference works under the stringent condition that only partial knowledge of the underlying SCM is available, and usually the information of structural dependency(i.e. causal diagrams) might be sufficient for many causal effect estimation problems. Particularly, \emph{do}-calculus is available for identification of causal queries when the causal structure is known, while its graphical criteria are placed on graphs which obtained by removing input/output edges on intervened variables from the original causal graph. Since every hypothetical intervention is creating a counterfactual world for domain variables, then we can define a causal graph corresponding to the modified mechanisms. We here introduce a novel graphical tool \emph{intervention diagram} 
to further illustrate how the path intervention works.

\begin{definition}[Intervention diagram]
    For an intervention, which can be \emph{do}, \emph{info} or \emph{path} intervention, the causal diagram for the modified mechanisms induced by the intervention is called the \emph{intervention diagram (or graph)}.\footnote{Note that intervention diagrams are not restricted to path intervention. For some cases, we may also want to include exogenous variables to represent the causal mechanisms more detailed graphically, we call the causal graph of the intervention model with both exogenous and endogenous variables the \emph{augmented intervention diagram}. }
\end{definition}

To emphasis the informational accounts of causality and differences from previous causal graphs, our intervention diagrams, which is created based on the modified mechanisms by an intervention on SCM, will use lowercase letter to represent variables. The following toy example is used to present the intervention diagrams for different interventions. 

\begin{example}
	\label{eg2}
	An SCM $\mathcal{M}$ with endogenous variables treatment $T$, outcome $Y$, confounder $Z$ and theire corresponding exogenous variables $U_T, U_Z, U_Y$ and an element $t'$ in the domain of $T$, and its causal mechanisms represented by structural equations(the corresponding augmented causal graph is Fig\,\ref{fig:interventions}(a)):
	$$
	\begin{aligned}
	\begin{cases}
	z \leftarrow f_Z(u_Z),  \\   
	t \leftarrow f_T(z, u_T), \\
	y \leftarrow f_Y(t, z, u_Y).
	\end{cases}
	\end{aligned}
	$$
	Its do-intervention SCM $\mathcal{M}^{do(t')}$ with augmented intervention diagram Fig\,\ref{fig:interventions}(b) has the following modified structural equations:
	$$
	\begin{aligned}
	\begin{cases}
	z \leftarrow f_Z(u_Z),  \\   
	t \leftarrow t', \\
	y \leftarrow f_Y(e_{TY}, e_{ZY}, u_Y); e_{TY} \leftarrow t, e_{ZY} \leftarrow z.
	\end{cases}
	\end{aligned}
	$$	
    Its info-intervention SCM $\mathcal{M}^{\info(t')}$ with the augmented intervention diagram Fig\,\ref{fig:interventions}(c) has the following modified structural equations:
	$$
	\begin{aligned}
	\begin{cases}
	z \leftarrow f_Z(u_Z),  \\
	t \leftarrow f_T(e_{ZT}, u_T); e_{ZT} \leftarrow z, \\
	y \leftarrow f_Y(e_{TY}, e_{ZY}, u_Y); e_{TY} \leftarrow t', e_{ZY} \leftarrow z.
	\end{cases}
	\end{aligned}
	$$    	
    Its path-intervention SCM $\mathcal{M}^{\pi(t')}$(for causal path $\pi:T \to Y$ with input information $T=t'$) with augmented intervention diagram Fig\,\ref{fig:interventions}(d) \footnote{Since $T^\pi, Z^\pi$ are not affected by the information $T=t'$ through the path $\pi$, then we marginalize out the mechanisms for them.} has the following modified structural equations:
	$$
	\begin{aligned}
	\begin{cases}
	z \leftarrow f_Z(u_Z),  \\   
	t \leftarrow f_T(e_{ZT}, u_T); e_{ZT} \leftarrow z, \\
	y \leftarrow f_Y(e_{TY}, e_{ZY}, u_Y); e_{TY} \leftarrow t, e_{ZY} \leftarrow z, \\
	y^{\pi} \leftarrow f_Y(e'_{TY}, e'_{ZY}, u'_Y); e'_{TY} \leftarrow t', e'_{ZY} \leftarrow z.
	\end{cases}
	\end{aligned}
	$$ 
\end{example}

\begin{figure}[t]
\begin{center}
\includegraphics[width=0.8\linewidth]{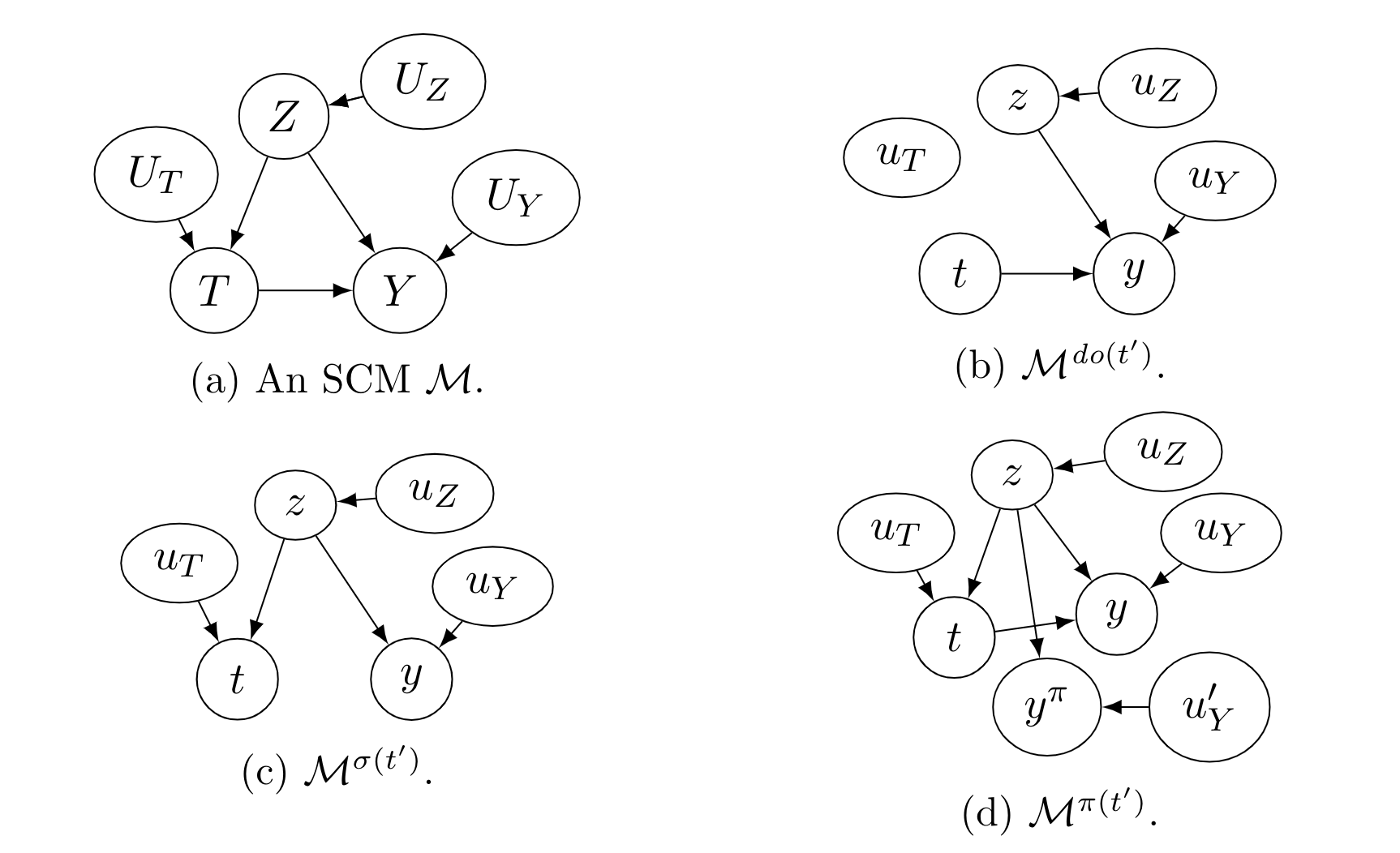}
\end{center}
\caption{Causal graphs of an SCM and three augmented intervention diagrams of its intervention SCMs.}
\label{fig:interventions}
\end{figure}

For conciseness, we usually construct the (augmented) causal graph of those intervention SCMs without separating the variables in the counterfactual world created by the hypothetical intervention from the factual world, except for path intervention. In particular, the counterfactual variable $Y^{do(t')}$ created by do intervention $do(t')$ is represented with notation $y$ instead of $y^{do(t')}$. Henceforth, we are using $y$ rather than $y^{do(t')}$ or $y^{\info(t')}$ in those structural equations in the formal definition of \emph{do} intervention and \emph{info} intervention. However, we find that the \emph{path} intervention is quite different from \emph{info} or \emph{do} intervention in that it includes cross-world structural equation between two endogenous variable represented by $z$ and $y^\pi$. In fact, the cross-world counterfactuals involved in mediation analysis is exactly what makes it more complicated, and we need the counter-intuitive recanting witness criteria for identification with traditional methods. In contrast, our novel path specific effects exhibit intuitive interpretation and identification properties. 



\section{Identification of Path-specific Effects}
\label{sec:ident}

\citet{robins2020interventionist} defines the potential outcome $V_i(\pi,a,a')$ by setting $A$ to $a$ for path $\pi$ that end in $V_i$, and setting $A$ to $a'$ for the purposes of proper causal paths from $A$ to $V_i$ other than $\pi$. Formally, the definition is as follows, for any $V_i \in V$:
\begin{align}
\label{eqn:pse}
V_i(\pi, a, a') &\equiv a \;\text{ if }\; V_i \text{ is } A,\\
V_i(\pi, a, a') &\equiv 
V_i( \{ V_j(\pi, a, a') \mid V_j \in pa^{\pi}_i \}, \{ V_j(a') \mid V_j \in pa^{\overline{\pi}}_i \} )
\notag
\end{align}

where $V_j(a') \equiv a'$ if $V_j \text{ is } A$, and given by recursive substitution otherwise,  
$pa^{\pi}_i$ is the set of parents of $V_i$ along the path $\pi$, and
$pa^{\overline{\pi}}_i$ is the set of all other parents of $V_i$. Their path specific effects rely on the recursive assumption, and no recanting witness for identification. Although our path-specific effects can be defined for non-recursive SCMs, but we will not further address that problem in this paper considering the numerous complications for cyclic causal models. 

Identification for path-specific effects based on nested counterfactuals for causal graphical models is governed by the criterion known as the recanting witness criterion \citep{avin2005identifiability, shpitser2016causal, shpitser2020multivariate}. Here let's look into the causal diagrams Fig. \ref{fig:base}(a), a submodel of $X, M, Y$ can be induced through marginalization with the causal graph $X \to M \to Y$ and $M \to Y$. This is the simplest case of mediation analysis, hence the path specific effect along $\pi: X \to Y$ can be identified. However, the corresponding path $\pi: X \to A \to Y$ specific effect can not be identified because $A$ is a recanting witness. This leads to a model with additional observational information of $A$ to disable the identifiability of path-specific effect which conceptually contradicts the intuition that the more information the better. 

In contrast, our novel path-specific effects induced by our path intervention can be identified regardless of whether it satisfies recanting witness criteria or not. In particular, for a path intervention $\pi(a')$, the joint distribution of counterfactual variables can always be identified by the $\pi$-formula. To formulate our result, we denote $desc^\pi(A)$ as descendants of $A$ in the path $\pi$, i.e., all nodes in $\pi$ with a parent node in causal path $\pi$.

\begin{theorem}[$\pi$-formula]
In a causal diagram $\cal G$(of a Markovian \SCM $\cal M$) with a factorized probability $p$, for a causal path $\pi: A \to \dots \to Y$, the joint counterfactual distribution over variables $p(v^\pi, v)$ for path intervention $\pi(A=a')$ is identified via a equation called the $\pi$-formula:
\begin{equation}
\begin{aligned}
    p(v^\pi_{desc^\pi(A)}, v) &= \prod_{k \in desc^\pi(A)} p(v^{\pi}_k|e_{pa(k) , k}^\pi) \cdot \prod_{i \in V} p(v_i | v_{pa(i)})
\end{aligned}
\end{equation}
where $e^{\pi}_{j, k} = v_j \, \text{if} \, (j, k) \notin \pi ;\, \text{ else if } V_j \text{ is }  $A$ \,,  a'; \text{else} \, v^{\pi}_{j}$.
\end{theorem}

\begin{proof}
By definition of path intervention, the modified causal mechanisms for $(\mathbf{V}^\pi, \mathbf{V})$ are, for any  $i \in V$:

\begin{equation*}
\left\{
\begin{aligned}
v_{i} &\leftarrow  f_i({e}_{pa(i), i}, u_i) \\
e_{j, i} &\leftarrow v_j \\
v^{\pi}_{i} &\leftarrow  f_i(e'_{pa(i), i}, u'_i) \\
e'_{j, i} &\leftarrow v_j \, \text{if} \, (j, i) \notin \pi ;\, \text{else if } V_j \text{ is } A \,,  a'; \text{else} \, v^{\pi}_{j} \\
\end{aligned}
\right.
\end{equation*}

The intervened SCM is also acyclic with independent errors, thus factorize with the following density:

\begin{equation*}
\begin{aligned}
    p(v^\pi, v) &= \prod_{k \in V} p(v^{\pi}_k|e_{pa(k) , k}^\pi) \cdot \prod_{i \in V} p(v_i | v_{pa(i)})
\end{aligned}
\end{equation*}

where $e^{\pi}_{j, k} = v_j \, \text{if} \, (j, k) \notin \pi ;\, \text{ else if } V_j \text{ is }  $A$ \,,  a'; \text{else} \, v^{\pi}_{j}$. Notice that for any counterfactual variable not in $desc^\pi(A)$, it receives the information directly from factual variables and the i.i.d. copy of exogenous variable. These counterfactual variables are leafs of the intervention diagram, hence marginalizing out them directly leads to the $\pi$-formula.
\end{proof}

We then can formally identify the path specific effect for causal diagram Fig. \ref{fig:base}(a) in the following example:

\begin{example}
\label{eg:recanting}
For a causal diagram(of a Markovian \SCM $\cal M$) for variables $X, A, M, Y$, with causal relationships $X \to A \to M \to Y$ and $A \to Y$ in Fig \ref{fig:base}(a), the joint distribution of counterfactual and factual variables in intervention diagram Fig. \ref{fig:ident1} of path intervention $\pi(x')$, where $\pi: X \to A \to Y$, can be identified the following equation:
$$
\begin{aligned}
    &p(y^\pi, a^\pi, y, m, a, x) = p(y^\pi|a^\pi, m) p(a^\pi|x')  p(y|a, m) p(m|a) p(a|x) p(x)
\end{aligned}
$$ 

\noindent
Then the density of effect variable $Y^\pi$ can be obtained by summation over $\{a^\pi, y, m, a, x\}$:
$$
\begin{aligned}
    p(y^\pi) &= \sum_{a^\pi, m, a, x} p(y^\pi, a^\pi, y, m, a, x) \\
    &= \sum_{m, a^\pi, a, x}  p(y^\pi|a^\pi, m) p(a^\pi|x')  p(m|a) p(a|x) p(x) \\
\end{aligned}
$$

\end{example}

\begin{figure}[http]
\label{}
	\begin{center}
		\begin{tikzpicture}[>=stealth, node distance=2cm]
		\tikzstyle{format} = [draw, thick, circle, minimum size=6.0mm,
		inner sep=1pt]
		\tikzstyle{unode} = [draw, thick, circle, minimum size=4.0mm,
		inner sep=0pt, outer sep=0.9pt]
		\tikzstyle{square} = [draw, very thick, rectangle, minimum size=4mm]
		\path[->,  line width=0.9pt]
		node[format, shape=ellipse] (a) {$a$}
		node[format, shape=ellipse, above right of=a] (m) {$m$}
		node[format, shape=ellipse, right of=a] (y) {$y$}
		node[format, shape=ellipse, left of=a] (x) {$x$}
		node[format, shape=ellipse, above right of=x] (ac) {$a^\pi$}
		node[format, shape=ellipse, above right of=ac] (yc) {$y^\pi$}
        (x) edge[red] (a)
		(a) edge[blue] (m)
		(m) edge[blue] (y)
		(a) edge[red] (y)
		(ac) edge[red] (yc)
		(m) edge[blue] (yc)
		;
		\end{tikzpicture}
	\end{center}
\caption{Intervention diagram induced by path intervention $\pi(x')$ , where $\pi:  X \to A \to Y$ with input information $X=x'$.
\label{fig:ident1}}
\end{figure}
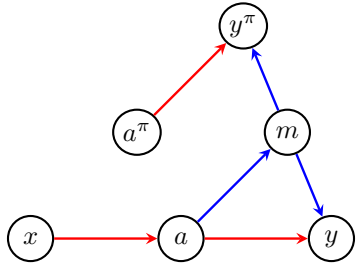


Let’s see a more complicated example in which $M$ is a recanting witness.

\begin{example}
For a causal diagram(of a Markovian \SCM $\cal M$) with a probability $p$ that factorize, we consider the path intervention $\pi(a')$ for the causal path $\pi: A \to M \to Y$ with input information $A=a'$. Then according to intervention graph Fig. \ref{fig:ident2}(b) with $\pi$-formula, we have
$$
p(y^\pi, m^\pi, z, m, a, w) = p(y^\pi|m^\pi, w, z) p(y^\pi|w, a') \cdot p(z|m, a) p(m|a, w) p(a|w) p(w)
$$

Note that we use latent projection onto variables of interest. In particular, we exclude the factual variable $Y$ in the intervention graph for conciseness,  which can be marginalized over the $\pi$-formula. However you cannot marginalize out factual variable $M$ to identify the distribution of $Y^\pi$.

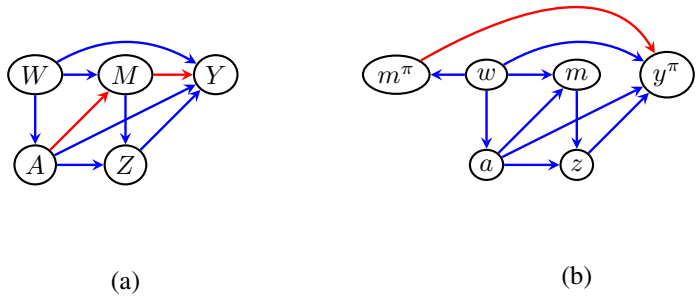
\begin{figure}[http]
	\begin{center}
		\begin{tikzpicture}[>=stealth, node distance=1.2cm]
		\tikzstyle{format} = [draw, thick, circle, minimum size=4.0mm,
		inner sep=1.8pt]
		\tikzstyle{unode} = [draw, thick, circle, minimum size=1.0mm,
		inner sep=0pt,outer sep=0.9pt]
		\tikzstyle{square} = [draw, very thick, rectangle, minimum size=4mm]
		\begin{scope}[xshift=0cm]
		\path[->,  line width=0.9pt]
		node[format, shape=ellipse] (a) {$A$}
		node[format, shape=ellipse, above of=a] (w) {$W$}	
		node[format, shape=ellipse, right of=w] (m) {$M$}	
		node[format, shape=ellipse, below of=m] (z) {$Z$}	
		node[format, shape=ellipse, right of=m] (y) {$Y$}
		node[below=1.0cm of z]{(a)}	
		(w) edge[blue] (a)
		(w) edge[blue] (m)
		(w) edge[blue,out=30,in=150] (y)
		(a) edge[red] (m)
		(a) edge[blue] (z)
		(a) edge[blue] (y)
		(m) edge[red] (y)
		(m) edge[blue] (z)
		(z) edge[blue] (y)
		;
		\end{scope}
        \begin{scope}[xshift=6cm]   
		\path[->,  line width=0.9pt]
		node[format, shape=ellipse] (a) {$a$}
		node[format, shape=ellipse, above of=a] (w) {$w$}	
		node[format, shape=ellipse, right of=w] (m) {$m$}		
		node[format, shape=ellipse, left of=w] (mpi) {$m^{\pi}$}	
		node[format, shape=ellipse, below of=m] (z) {$z$}	
		node[format, shape=ellipse, right of=m] (y) {$y^{\pi}$}
		node[below=1.0cm of z]{(b)}	
		(w) edge[blue] (a)
		(w) edge[blue] (m)
		(w) edge[blue] (mpi)
		(w) edge[blue,out=30,in=150] (y)
		(a) edge[blue] (m)
		(a) edge[blue] (z)
		(a) edge[blue] (y)
		(mpi) edge[red,out=30,in=120] (y)
		(m) edge[blue] (z)
		(z) edge[blue] (y)
		;
		\end{scope}
	    \end{tikzpicture}
	\end{center}
\caption{
(a) A causal diagram with recanting witness $M$; (b) Intervention diagram.
\label{fig:ident2}}
\end{figure}

\end{example}

This way of constructing intervention diagram is somewhat similar to twin-network \citep{pearl2009causal} for both including factual and counterfactual variables. However, there are completely different in several ways. First, the twin-SCM only have $U_V$ which connects factual and counterfactual variables, our intervention diagram includes an i.i.d. copy of $U_V$ for the corresponding intervened SCM.  Second, there are no direct edges between factual and counterfactual variables in a twin-SCM. In contrast,  there are no direct edges between two counterfacutal variables unless there exists an edge in $\pi$ connecting them, which leads to counterfactual variables usually have no output edge in the intervention diagram. In other words, parents of counterfactual variables are usually factual variables, e.g. $A, Z, W \to Y^\pi$ in Fig. \ref{fig:ident2}(b).
Third, Our intervention diagrams include probabilistic relations instead of deterministic relations, while the incompleteness of d-separation occurs in twin-SCM due to deterministic relations \citep{shpitser2020multivariate}.


\section{Concluding remarks and Discussion}
\label{sec:conclude}

In this paper, by relating the informational account to \SCM, we propose a novel \emph{path} intervention framework to formalize path-specific effects aligned the dictum ``no causation without manipulation''. In the \emph{path} intervention framework, the information account of causality has suggested of separating causal mechanisms into information transfer and process. Our path intervention explicitly manipulate the mechanisms, thus path-specific effects even can be well-defined for cyclic \SCM s. Intuitively, our path-specific effect $Y^{\pi(a')}$ is obtained by making the treatment information $A = a'$ exclusively pass through the causal path $\pi$ while allowing all other information outside from the $\pi$ to keep constant. Moreover, we address the identifiability of path-specific effects for causal diagrams and show that they can be identified even when recanting witness exists with the $\pi$-formula.


\paragraph{More discussions.} Causality has been one of the basic topics of philosophy since the time of Aristotle. However during the last two decades or so, interest in causality has become very intense in the philosophy of science community, and a great variety of novel views on the subject have emerged and been developed. Among those novel views, it has been claimed that an informational account of causality, formally proposed and developed since \citet{collier1999causation} in philosophy, might be useful to the scientific problem of how we think about, and ultimately trace, causal linking, and so to causal inference and reasoning. An informational account of causality can be useful to help us reconstruct how science builds up understanding of the causal structure of the world. Broadly, we find mechanisms that help us grasp causal linking in a coarse-grained way. Then we can think in terms of causal linking in a more fine-grained way by thinking informationally. An informational account of causality may also give us the prospect of saying what causality is, in a way that is not tailored to the description of reality provided by a given discipline. And it carries the advantage over other causal metaphysics that it fares well with the applicability problem for other accounts of production (processes and mechanism)\citep{illari2014causality}. But from an application perspective, what benefits we can gain in causal modeling? Or it is just a rhetorical flourish? 

Algorithmic information theory has been used in causal modeling(see e.g. \citep{scholkopf2019causality}).It is also worth to mention that, broadly speaking, the information accounts of causality can also facilitate interpretation of existing widely used causal propositions. For example, regarding back-door/front-door criteria, the goal of which can be consistently considered as whether the observational information of a set of variables is enough or not to answer causal-effect estimation question, instead of conventional understanding such as controlling variables. Moreover, in general sense of information accounts, Pearl points out that questions in one layer of the causal hierarchy can only be answered when corresponding layer information are available\citep{pearl2019seven, bareinboim2020pearl}, and Scholk\"opf believes causal science will enable us act and decision with information from Lorenzian imagined space \citep{scholkopf2019causality}. The recently proposed Mini-Turing test for AI --- How can machines represent causal knowledge in a way that would enable them to access the necessary information swiftly, answer questions correctly, and do it with ease, as a human can? \citep{pearl2018book}. In summary, to build true intelligent machines, climb the ladder of data, information, knowledge and wisdom, we might need to incorporate the information accounts of causality into causal tasks.

\bibliography{egbib.bib}
\bibliographystyle{icml2021}





\end{document}